\documentclass[11pt]{article}
\usepackage{jmlr2e_g}
 \usepackage{algorithm}
\usepackage{amsmath}
 \newtheorem{claim}{Claim}

\DeclareMathOperator{\w}{\mathbf{w}}
\DeclareMathOperator{\W}{\mathbf{W}}
\DeclareMathOperator{\x}{\mathbf{x}}
\DeclareMathOperator{\X}{\mathbf{X}}
\DeclareMathOperator{\K}{\mathbf{K}}
\DeclareMathOperator{\Y}{\mathbf{Y}}
\DeclareMathOperator{\A}{\mathbf{A}}
\DeclareMathOperator{\B}{\mathbf{B}}
\DeclareMathOperator*{\minimize}{\mathrm{minimize}}

\firstpageno{1}

\begin{document}

\title{KNIFE: Kernel Iterative Feature Extraction}

\author{\name Genevera I. Allen \\  \email giallen@stanford.edu \\
  \addr Department of Statistics \\
  Stanford University \\
  Stanford, CA 94305-4065, USA
}


\maketitle

\begin{abstract}
Selecting important features in non-linear or kernel spaces is a difficult
challenge in both classification and regression problems.  When many
of the features are irrelevant, kernel methods such as the support
vector machine and kernel ridge regression can sometimes perform poorly.  We
propose weighting the features within a kernel with a sparse set of
weights that are estimated in conjunction with the original
classification or regression problem.  The iterative algorithm, KNIFE,
alternates between finding the coefficients of the original problem
and finding the feature weights through kernel
linearization.  In addition, a
slight modification of KNIFE yields an efficient algorithm for finding
feature regularization paths, or the paths of each feature's weight.
Simulation 
results demonstrate the utility of KNIFE for both kernel regression and
support vector machines with a variety of kernels.  Feature path
realizations also reveal important non-linear correlations among
features that prove useful in determining a subset of significant
variables.  Results on  vowel recognition data, Parkinson's disease
data, and microarray data are also given.
\end{abstract}


\section{Introduction}

Selecting important features with kernel regression and classification
methods is a challenging problem.   With linear
problems, however, 
several efficient feature selection methods exist.  These include
penalization methods, such as the lasso, elastic net and $L_{1}$ SVM
and logistic regression, subset methods, such as all subsets, forward
and backward elimination, and filtering methods such as correlation
and $t$-test filtering.  These types of feature selection techniques
also have analogous versions for kernel methods.  Filtering methods and subset
methods, especially the popular Recursive Feature Elimination
\citep{rfe} which
removes features in a backwards stepwise manner, are the
most commonly used methods.  In addition, several penalization methods
exist for the linear SVM \citep{neumann, l1svm}, and some also that
can be adapted for kernel SVMs \citep{wang, margin_feat, grandvalet,
  weston, guyon_multivariate}.  Many of these methods, however, are
computationally 
intensive and only applicable to the support vector machine.

Instead, we propose a penalization method that aims to extract
significant features by finding a sparse set of feature weights within
a kernel in conjunction with estimation of the model parameters.
Weights within
kernels are the key to our problem formulation and have been proposed
in several feature selection techniques \citep{cao, fvm, grandvalet,
  margin_feat, dc_alg, weston}.  Many of these methods do not directly
optimize the original regression or classification problem, but
instead seek to find a good set of weights on the features for later
use within the kernel of the model.  
\citet*{grandvalet}, however, formulate an optimization problem for the
support vector machine which iteratively optimizes the SVM criterion
for the coefficients and then weights within kernels.  They place an
$L_{p}$ penalty on the weights, noting that the optimization problem
is extremely non-convex for $p < 2$, and thus use $p = 2$ for convenient
computation.  An $L_{2}$ penalty, however, does not encourage
sparsity in the feature weights, and so we propose a similar optimization problem with an $L_{1}$
penalty that can be applied not only to the support vector machine,
but any kernel classification or regression problem.

In this paper, we present an algorithm that selects important features in
kernel methods: KerNel Iterative Feature Extraction (KNIFE).  First, in
Section \ref{section_knife_problem} we
discuss weighted kernels and the
 KNIFE optimization problem along with its mathematical challenges. 
Then we present our main algorithm, KNIFE, in Section
\ref{section_knife} discussing minimization
through kernel linearization.  We also discuss KNIFE for linear
kernels and its connections with other
previously proposed regression and classification methods along with
convergence results.  A
path-wise version of KNIFE is presented to determine weighted feature
paths in Section \ref{section_path_alg}.
Section \ref{section_results} gives both simulation results and
feature path realizations for kernel regression and support vector
machines with a variety of kernel types, along with examples of gene
selection for microarray data and feature selection in vowel
recognition and Parkinson's disease data.  We conclude with a
discussion, Section \ref{section_discussion}, giving possible
applications of KNIFE and future improvements for high-dimensional
settings.

\section{KNIFE Problem}
\label{section_knife_problem}

We propose to select important features by forming a penalized loss
function that involves a set of weights on the features within a
kernel.  Before presenting what we term feature weighted kernels and
the KNIFE optimization problem, we present an example of the need for
feature selection in non-linear spaces.

\subsection{Motivating Example}

To motivate the necessity of feature selection in kernels,
we present a classification example using support vector
machines.
Here, data is simulated under the skin of the orange simulation in
which there are four true features \citep{esl}.  Class one has four
standard 
normal features, $\x_{1}, \x_{2}, \x_{3}, \x_{4}$,
and the second class has the same conditioned on $9 \leq
\sum_{j=1}^{4} \x_{j}^{2} \leq 16$. Thus, the second class surrounds the
first class like the skin of an orange.  The two classes are not
completely separable, however, giving a Bayes error of 0.0611.  We
present a simulation where we add noise features to this model,
training the data on a set with 100 observations and reporting the test
misclassification error on a set with 1000 observations.  An SVM with
second order polynomials is used and parameters were selected using a
separate validation set.  These results are given in Figure
\ref{fig_svm_example_only}.  
From this skin of the orange example, we see that while SVMs perform
well when all the features are relevant, performance diminishes
greatly as more noise features are added.

\begin{figure}[!ht]
\begin{center}
    \includegraphics[width=2.5in]{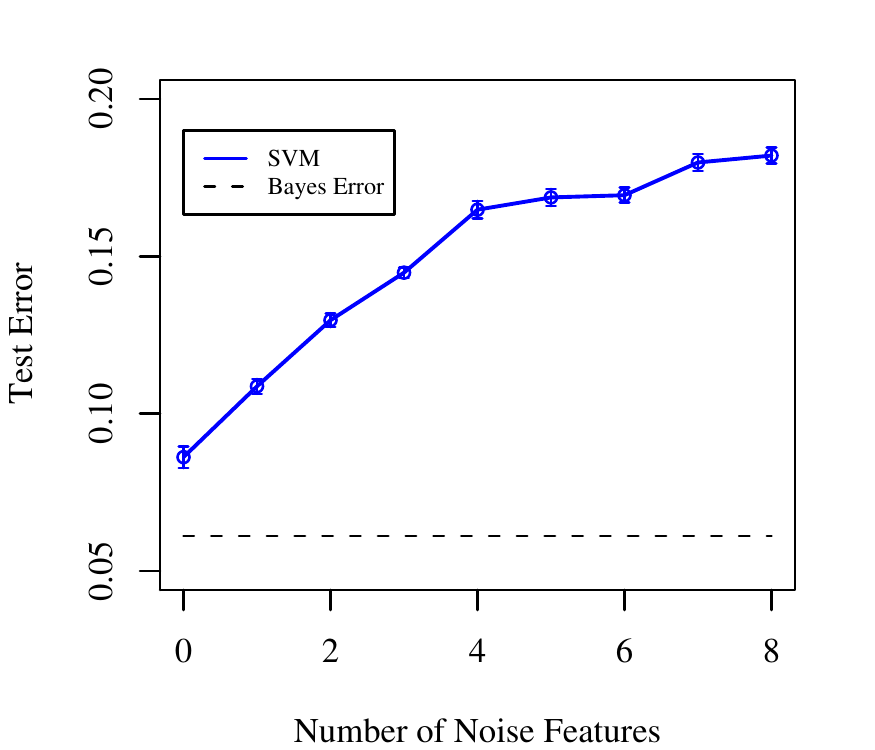}
\caption{ \it Mean test misclassification error when noise features are
  added to the skin of the orange simulation.  Support vector machines
 with second order polynomial kernels are trained on data with 100
 observations and tested on data with 1000 observations.  Ten
 simulations are conducted for each new noise feature.  The Bayes
 error is for this model is 0.0611.}
\label{fig_svm_example_only}
\end{center}
\end{figure}

We look at the support vector machine with non-linear kernels
mathematically to gain an understanding of why irrelevant features
have such an impact on the method.  
Given data $\x_{i}
\in \Re^{p}$ for $i = 1 \ldots n$ observations and $p$ features with
the response $\Y \in \{ -1, 1\}^{n}$.   The kernel
matrix, $\K_{\w} \in \Re^{n \times n}$ is defined by $\K_{\w}(i,i') =
k_{\w}(\x_{i}, \x_{i'}) =  \left( \sum_{j=1}^{p} x_{ij} x_{i'j}
  + 1 \right)^{d}$, for example with polynomial kernels.  Recall that
the support vector machine can be written as an unconstrained
minimization problem with the hinge loss \citep{wahba}.
\begin{align*}
\minimize_{\alpha, \alpha_{0}}  \hspace{4mm} \sum_{i=1}^{n} \left[ 1 - y_{i}
  \alpha_{0} - y_{i} (\K \alpha )_{i} \right]_{+} + \lambda \alpha^{T}
\K  \alpha.
\end{align*}
Here, the coefficients $\alpha$ determine the support vectors which
are sparse in the observations.  Each feature, however, has the same
impact on the objective since they are all given equal weight in the
kernel matrix, thus explaining the poor performance with noise
features.  We, therefore, propose to place feature weights within the
kernels to differentiate between true and noise features.

\subsection{Feature Weighted Kernels}
\label{section_weighted_kernels}

Before introducing feature weighted kernels, we give the data format.
The data, as previously mentioned, $\x_{i} \in \Re^{p}$ for $i=1
\ldots n$ observations and can be written as the data matrix $\X \in
\Re^{n \times p}$. We assume that $\x_{i}$ is standardized so
that it has mean zero and variance one.  For regression, the response
$\Y \in \Re^{n}$ and for classification, $\Y \in \{ -1, 1\}^{n}$.  
In the previous example, we gave an example of a typical kernel, the
polynomial kernel.  In this section, we present feature weighted
kernels, giving example for many common kernel types below.  These
kernels simply place a weight $\w \in \Re^{p+}$ on each feature in the
kernel.  

\begin{itemize}
\item {\it Inner Product Kernel:} \hspace{3mm}
$k_{\w}(\x,\x') = \displaystyle\sum_{j=1}^{p} w_{j} x_{j} x'_{j}$
\item {\it Gaussian (Radial) Kernel:} \hspace{3mm}
$k_{\w}(\x,\x') = \mathrm{exp} \left( -\gamma \displaystyle\sum_{j=1}^{p} w_{j}( x_{j}  - x'_{j})^{2} \right)$
\item {\it Polynomial Kernels:} \hspace{3mm}
$k_{\w}(\x,\x') = \left( \displaystyle\sum_{j=1}^{p} w_{j} x_{j} x'_{j} \right)^{d},
\hspace{4mm} k_{\w}(\x,\x') = \left( \displaystyle\sum_{j=1}^{p} w_{j} x_{j} x'_{j}
  + 1 \right)^{d}$
\end{itemize}

With these feature weighted kernels, we can define the kernel matrix
and loss function for a kernel prediction method as $\K_{\w} \in
\Re^{n \times n}$ such that $\K_{\w}(i,i') = k_{\w}(\x_{i}, \x_{i'})$.

\subsection{KNIFE Optimization Problem}

We incorporate these feature weighted kernels into the regression or
classification model.  The response, $\Y$ is modeled by $\Y = f(\X)$,
where $f(\x_{i}) = \sum_{i'=1}^{n} \alpha_{i} \K_{w}(i,i')$ for
regression or $f(\x_{i}) = \mathrm{sign}( \alpha_{0} + \sum_{i'=1}^{n}
\alpha_{i} \K_{w}(i,i') )$ for classification, and
$\alpha \in \Re^{n}$ are the coefficients that must be
estimated.  For a positive definite kernel, $\K_{\w}$, and $f(\X)$ a member
of the reproducing kernel Hilbert space, $\mathcal{H}_{K_{\w}}$, this
problem can be written as a minimization problem of the form
$\minimize_{f \in \mathcal{H}_{K_{\w}}} \left[ L(\Y, f(\X)) + \lambda
  || f ||^{2}_{\mathcal{H}_{K_{\w}}} \right]$, where $L(\Y, f(\X))$ is
the loss function.
Some common examples of these include the hinge loss (support
vector machine), squared error
loss (regression) or  binomial deviance loss (logistic regression).

To obtain a selection of
important variables in a problem of this form, we need the weights to be
both non-negative and 
sparse.  To this end, we propose adding an $L_{1}$ penalty on the
weights to induce sparsity and optimizing over the feasible set of
non-negative weights that are less than one.  This gives the KNIFE
optimization problem as stated below.
\begin{align}
\label{true_objective}
\minimize_{\alpha, \w} \hspace{5mm} & f(\alpha, \w)
= L(\Y, \K_{\w} \alpha ) + \lambda_{1}
\alpha^{T} \K_{\w} \alpha + \lambda_{2} \mathbf{1}^{T} \w \nonumber \\
\textrm{subject to} \hspace{4mm} & 0 \leq w_{j} < 1, \textrm{ for
  all } j = 1 \ldots p.
\end{align}

The KNIFE optimization problem, \eqref{true_objective}, is non-convex and it is therefore extremely
difficult to find a minimum, even for problems of small dimensions.
To illustrate the non-convexities and understand our approach to
finding a minimum, we offer an example.  
For SVMs with a Gaussian kernel, our optimization problem is
given by the following.
\begin{align}
\label{objective_svm}
\mathrm{minimize} \hspace{6mm}  & \sum_{i=1}^{n} \left[ 1 - y_{i} \alpha_{0}
 - y_{i} \sum_{i'=1}^{n} \alpha_{i'} \mathrm{exp} \left( - \gamma
    \sum_{j=1}^{p} w_{j} ( x_{ij} - x_{i'j} )^{2} \right)  \right]_{+}
\nonumber \\
& \qquad + \lambda_{1} \sum_{i=1}^{n} \sum_{i'=1}^{n} \alpha_{i} \alpha_{i'} \mathrm{exp} \left( - \gamma
    \sum_{j=1}^{p} w_{j} ( x_{ij} - x_{i'j} )^{2} \right) +
  \lambda_{2} \mathbf{1}^{T} \w \nonumber \\
\mathrm{subject \hspace{1mm} to} \hspace{4mm} & 0 \leq w_{j} < 1, \textrm{ for
  all } j = 1 \ldots p.
\end{align}
Here, notice that if we fix the weights, $\w$, we simply have a convex
optimization problem that is equivalent to solving the SVM problem.  If we fix $\alpha$ and
$\alpha_{0}$, however, the problem is not convex in $\w$.  This
results from the fact that the coefficients $\alpha$ are both positive
and negative.  One could approach this as a difference of convex
programming problem, a direction taken in \citep{dc_alg} for a slightly
different problem.  This method of minimizing \ref{true_objective} is computationally prohibitive.

Thus, we propose an alternative algorithmic approach by using
iterative convexifications of the weights within the kernel.  This is
discussed in the next section where we present the KNIFE algorithm.

\section{KNIFE Algorithm}
\label{section_knife}

Given the KNIFE optimization problem based on the feature weighted
kernels, we propose an algorithm to minimize the penalized loss
function \eqref{true_objective} for any kernel classification or
regression method.  The algorithm alternates between minimizing with
respect to the coefficients $\alpha$ and the feature weights $\w$.
For non-linear kernels, we need to convexify the weights to obtain a
feasible optimization problem.  But, to understand the fundamentals of
the algorithm we first discuss KNIFE for linear kernels in Section
\ref{section_linear}.  We then present the algorithm for non-linear
kernels, in Section \ref{section_algorithm}, also discussing kernel
convexification.  Finally, we give connections to several other
regression and non-parametric methods, Section
\ref{section_connections} along with KNIFE solution properties and
convergence results, Section \ref{section_convergence}.

\subsection{Linear Kernels}
\label{section_linear}

We give the coordinate-wise KNIFE algorithm for linear kernels which
form the foundation of the algorithm for non-linear kernels.  With
linear kernels, the kernel matrix becomes $\K_{\w} = \X \W \X^{T}$ where
$\W = \mathrm{diag}(\w)$.  This gives the following objective function.
\begin{align*}
f(\alpha, \w) =  L(\Y, \X \W \X^{T} \alpha ) + \lambda_{1} \alpha^{T} \X
\W \X^{T} \alpha + \lambda_{2} \mathbf{1}^{T} \w.
\end{align*} 
Letting $\beta = \X^{T} \alpha$, we arrive at
\begin{align}
\label{obj_linear}
f(\beta, \w) = L(\Y, \X \W \beta) + \lambda_{1} \beta^{T} \W \beta + \lambda_{2}
\mathbf{1}^{T} \w.
\end{align}
Here, notice that $f(\beta, \w)$ is a bi-convex function of $\beta$
and $\w$, meaning that if we fix $\beta$, $f( \cdot, \w)$ is convex in
$\w$ and if we fix $\w$, $f(\beta, \cdot)$  is convex in $\beta$.

This biconvex property leads to a simple coordinate-wise algorithm for
minimization, minimizing first with respect to $\beta$ with $\w$ fixed
and then with respect to $\w$.  While this coordinate descent
algorithm is monotonic, meaning that each iteration decreases the
objective function $f(\beta, \w)$, it does not necessarily converge
to the global minimum.  If the objective function satisfies certain
smoothness conditions (discussed in Section
\ref{section_convergence}), then coordinate descent converges to a
stationary point \citep{tseng}.  

For linear kernels, the optimization problem is still non-convex but
satisfies certain convex properties, namely bi-convexity in the
weights and the coefficients.  This is the approach that we will take for
non-linear kernels, namely linearizing the kernels with respect to
the weights to obtain a surrogate function which is convex in $\w$ as
presented in the following section.

\subsection{KNIFE Algorithm}
\label{section_algorithm}

We linearize kernels with respect to the feature weights to obtain a
function convex in the weights and hence conducive to easy
minimization.  From the previous section, we saw that if the function
is linear in the feature weights, then we can apply a block coordinate-wise
algorithm, finding the coefficients and then finding the weights.
Thus for non-linear kernels, we need to linearize kernels in such a
way that leads to an algorithm minimizing the objective function
\eqref{true_objective}.

The linearized kernels are given by $\tilde{k}_{\w}$ as defined
below.
\begin{align}
\label{kern_linearization}
\tilde{k}_{\w^{(t)}}(i,i') =
  k_{\w^{(t-1)}}(i,i') + \bigtriangledown k_{\w^{(t-1)}}(i,i')^{T} ( \w^{(t)} -
  \w^{(t-1)} ).
\end{align}
Notice that $\tilde{k}_{\w}(i,i')$ is the linearization of the
$(i,i')^{th}$ element of the kernel matrix $\K_{\w}$.  Here,
$\w^{(t-1)}$ is the weight vector from the previous iteration.
Minimization is done with respect to $\w^{(t)}$.

Linearizing the kernels in this manner, however, is not ideal for two
reasons.  First, notice that in non-linear kernels, the weights are placed on
the cross products or squared distance of the data vectors within a
non-linear function.  Thus, we need to place the weights on the same scale as
the data.  (Note that reparameterizing the linear kernels in terms of
$\beta$ solves this problem for linear kernels.) Secondly,
this naive linearization is not conducive to developing a stable
minimization algorithm.   To achieve these objectives, we linearize kernels with respect to
$\w^{2}$ instead of $\w$.  The differences between these can be
clearly seen with an example of the gradient of a polynomial kernel with
squared feature weights,
\begin{align*}
\bigtriangledown k_{\w^{(t-1)}}(i,i')_{k} =  2 d w^{(t-1)}_{k}  x_{ik} x_{i'k} \left( \sum_{j=1}^{p}
  w^{2^{(t-1)}}_{j} x_{ij} x_{i'j}  + 1 \right)^{d-1}.
\end{align*}
Here, notice that the gradient is scaled by the weights of the
previous iteration, $\w^{(t-1)}$.  Thus, if several weights were
previously set to zero, the gradient in those directions is zero
meaning that the weights will remain zero in all subsequent
iterations of the algorithm.  This feature, first of all maintains
sparsity in the feature weights throughout the algorithm, and secondly
limits the number of directions in which the weight vector can move in
succeeding iterations.  The second attribute can be critical to
algorithm convergence (see Section \ref{section_convergence}).

Thus, for non-linear kernels, our KNIFE objective changes slightly to
allow for linearization of the kernels with squared feature weights.
\begin{align}
\label{obj_w2}
f(\alpha, \w) = L(\Y, \K_{\w^{2}} \alpha ) + \lambda_{1}
\alpha^{T} \K_{\w^{2}} \alpha + \lambda_{2} \mathbf{1}^{T} \w .
\end{align}
We present the KNIFE algorithm for non-linear
kernels in Algorithm \ref{alg_knife}.  
\begin{algorithm}[!h]
\caption{KNIFE Algorithm}
\label{alg_knife}
\begin{enumerate}
\item Initialize $\alpha^{(0)}$ and $\w^{(0)}$ where $0 < w^{(0)}_{j} < 1$
  for $j = 1 \ldots p$.
\item Minimize $f(\alpha^{(t-1)}, \w^{(t-1)})$  with respect to
  $\alpha^{(t-1)}$.
\item Linearize $\K_{\w^{2}}$ element-wise giving $f(\alpha^{(t)} , \tilde{\w}^{(t-1)})$.
\item Minimize $f(\alpha^{(t)} , \tilde{\w}^{(t-1)})$  with respect to
  $\tilde{\w}^{(t-1)}$ subject to $0 \leq w_{j} < 1$ for $j = 1
  \ldots p$.
\item Repeat steps 2-4 until convergence.
\end{enumerate}
\end{algorithm}

We will take a closer look at the KNIFE algorithm by presenting an
example with support vector machines.  Step 2 is an SVM problem
finding the coefficients. 
\begin{align*}
\mathrm{minimize} \hspace{2mm}  \sum_{i=1}^{n} \left[ 1 - y_{i} \alpha_{0} -
y_{i} (\K_{\w^{2}} \alpha)_{i} \right]_{+} + \lambda_{1} \alpha^{T}
\K_{\w^{2}} \alpha.
\end{align*}
Now, we define 
\vspace{-2mm}
\begin{align*}
\B \in \Re^{n \times n}: \hspace{2mm} 
&  \B_{ii'} = k_{\w^{2^{(t-1)}}}(i,i') - \bigtriangledown
  k_{\w^{2^{(t-1)}}}(i,i')^{T} \w^{(t-1)} \\
 \A \in \Re^{n \times
    p}:  \hspace{2mm} & \A_{ii'}   = \sum_{i'=1}^{n} \alpha_{i'} \bigtriangledown
  k_{\w^{2^{(t-1)}}}(i,i')^{T}.
\end{align*}
Then, the objective function in Step
  4 becomes  
\begin{align*}
\textrm{minimize  } \sum_{i=1}^{n} \left[ 1 - y_{i} \alpha_{0} -
    y_{i} (\B \alpha)_{i} - y_{i} (\A \w)_{i} \right]_{+} + \lambda_{1}
  \alpha^{T} \A \w + \lambda_{2} \mathbf{1}^{T} \w.
\end{align*}
Notice that Step
4 reduces to a loss function that is linear in $\w$.  In fact, for
any loss function, the objective function of Step 4 becomes $L(\Y,
\B \alpha  + \A \w)$.

Thus, the KNIFE algorithm iterates between finding the coefficients of
the regression or classification problem and then finding the sparse
set of feature weights.  With linearizations of the kernels with
respect to the squares of the weights, the iterative optimization to
find coefficients and weights are both problems of the same form but in
different spaces.  For example, with squared error loss, minimization
with respect to the coefficients is a form of least squares problem in
$n$ dimensional space whereas minimization with respect to the weights
in the linearized kernel is also a form of least squares problem in
$p$ dimensional feature space.  Now that we have presented the
algorithm framework, we go back to the skin of the orange example to
briefly demonstrate the KNIFE algorithm in action.

\begin{figure}[!ht]
\begin{center}
    \includegraphics[width=2.5in]{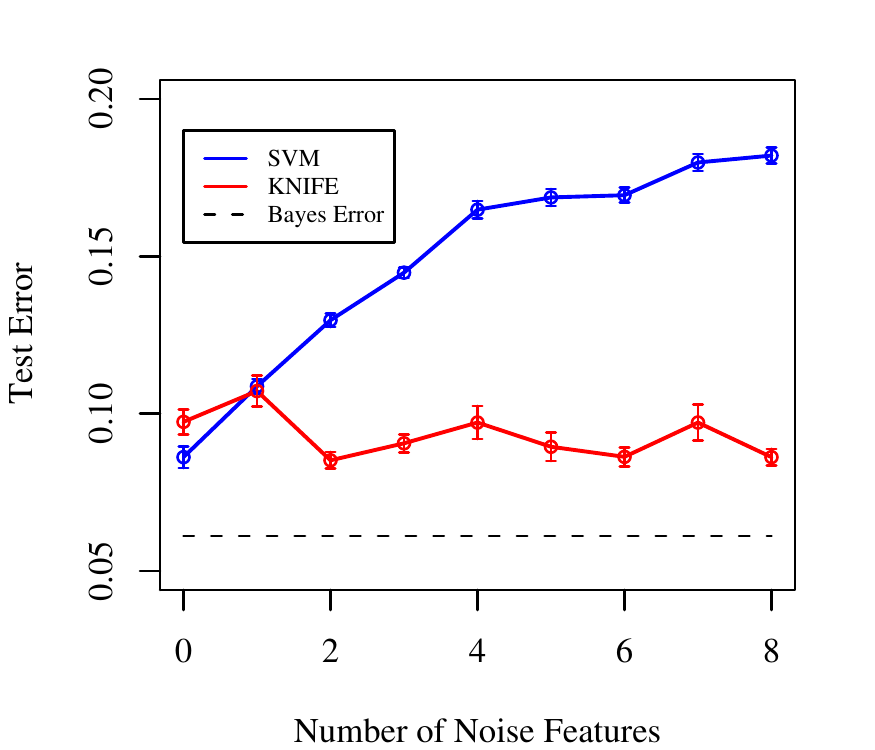}
\caption{ \it Mean test misclassification error when noise features are
  added to the skin of the orange simulation.  SVMs and KNIFE for SVMs
 with second order polynomial kernels are trained on data with 100
 observations and tested on data with 1000 observations.  Ten
 simulations are conducted for each new noise feature.  The Bayes
 error is for this model is 0.0611.}
\label{fig_svm_example}
\end{center}
\end{figure}

In Figure \ref{fig_svm_example}, we present the misclassification
error for the KNIFE as well as the SVM.  Here, we see that KNIFE
preforms well even when many noise features are added to the model.
To understand the reason for this performance improvement and further
examine the KNIFE problem and algorithm, 
we discuss several connections to existing
methods and present convergence results.

\subsection{Connections with Other Methods}
\label{section_connections}

While the optimization problem \eqref{true_objective} may appear
unfamiliar, there are several variations that are the same or similar
in form to many existing methods.  We first begin by looking at linear
kernels for regression problems with squared error loss.  A form of
the objective with linear kernels is given in the previous section,
\eqref{obj_linear}.  We give this and an equivalent form for comparison
purposes.  
\begin{align}
\label{lin_1}
\mathrm{minimize} \hspace{5mm} & || \Y - \X \W \beta ||_{2}^{2} +
\lambda_{1} \beta^{T} \W \beta + \lambda_{2}||\w ||_{1} \\
\label{lin_2}
\textrm{  or} \hspace{4mm} & || \Y - \X \tilde{\beta} ||_{2}^{2} + \lambda_{1}
\tilde{\beta}^{T} \W^{-1} \tilde{\beta} + \lambda_{2} ||\w ||_{1} \\
\textrm{subject to}\hspace{4mm} & \w \geq 0.  \nonumber
\end{align}

These are closely related to several common regression methods.
First, if we let $\lambda_{2} = 0$ and $\w = \mathbf{1}$ in
\eqref{lin_1}, we get ridge regression.  If we let $\lambda_{1} = 0$,
then we have the form of the non-negative garrote \citep{nng}.  The form of both
\eqref{lin_1} and \eqref{lin_2} is very similar to the elastic net
which places an $L_{1}$ and $L_{2}$ penalty on the coefficients
\citep{el_net}.  In 
the KNIFE, however, the $L_{1}$ penalty is not on the coefficients,
but on the weights that multiply the coefficients.  Letting
$\lambda_{1} = 0$, we get a problem very similar in structure and
intent to the lasso.  Also, if we let $\lambda_{2} = 0$, then we have
a problem that puts weights on the $L_{2}$ penalty on the
coefficients.  This is similar to the adaptive lasso which
places weights on the $L_{1}$ penalty on the coefficients
\citep{adalasso}.  

In addition to these methods, a special case of the COSSO
(Component Selection and Smoothing Operator) which estimates
non-parametric functions give \eqref{lin_1} and \eqref{lin_2}
exactly \citep{COSSO}.  This method in theory minimizes the sums of
squares between the response and a function with an $L_{2}$ penalty on
the projection of the function scaled by the inverse of a non-negative
weight.  This proposed theoretical form is also the form of
\eqref{lin_2}.  In addition, the COSSO employs an algorithm which
first fits a smoothing spline and then fits a non-negative garrote,
noting that these steps can be repeated.  This algorithmic approach is
also analogous to the KNIFE algorithm for the special case of squared
error loss with linear kernels.

These similarities between other regression methods and KNIFE hold
with other forms of loss functions also. 
For support vector machines, we have a problem similar to the $L_{1}$
and $L_{2}$ support vector machines in the same way that the
inner-product squared error loss KNIFE relates to ridge and lasso
regression.  The same is true of $L_{1}$ and $L_{2}$ regularized
logistic regression.

\subsection{Convergence of KNIFE}
\label{section_convergence}

In this section we discuss the convergence of the KNIFE algorithm and
the properties of the KNIFE solution,  also giving numerical
examples.
As previously discussed, the KNIFE objective is highly non-convex
and thus it is difficult to assert any claims on the convergence of
the KNIFE algorithm or the optimality of the KNIFE solution.  For
special cases, however, we can guarantee convergence of the KNIFE
algorithm to a stationary point.

\begin{claim}
\label{conv_claim}
If the KNIFE algorithm finds a unique minimum for the coefficients,
$\alpha$, and the weights, $\w$, in each step and the loss function and
kernel are continuously
differentiable, then the KNIFE algorithm monotonically decreases the
objective and converges to a stationary point of $f(\alpha, \w)$.  
\end{claim}
\begin{proof}
Differentiability of the loss function and kernel implies that
$f(\alpha, \w)$ is regular on its domain.  This along with unique
minima in both blocks of coordinates satisfies conditions for monotonic
convergence to a 
stationary point for non-convex functions.  In theory,
differentiability can be relaxed to weaker conditions for regularity
\citep{tseng}.   
\end{proof}

The conditions in Claim \ref{conv_claim} are satisfied with strictly
convex, differentiable loss functions with linear kernels.  An example
here is the squared error loss and linear kernel given in
\eqref{lin_1}.  We have noted that these examples are bi-convex, and
thus the KNIFE algorithm simply iterates between minimization with
respect to the coefficients and then the feature weights.  Hence, for
these special cases of KNIFE, we are guaranteed to converge to a
stationary point.  We must note, however, that for non-convex
functions, there can be potentially many stationary points.  Thus, the
stationary point at which KNIFE arrives will depend on the random
starting values of the weights.  For this reason, we recommend
initializing the KNIFE algorithm at several random starting points and
taking the solution which gives the minimum
objective value.

\begin{figure}[!ht]
\begin{center}
    \includegraphics[width=6in]{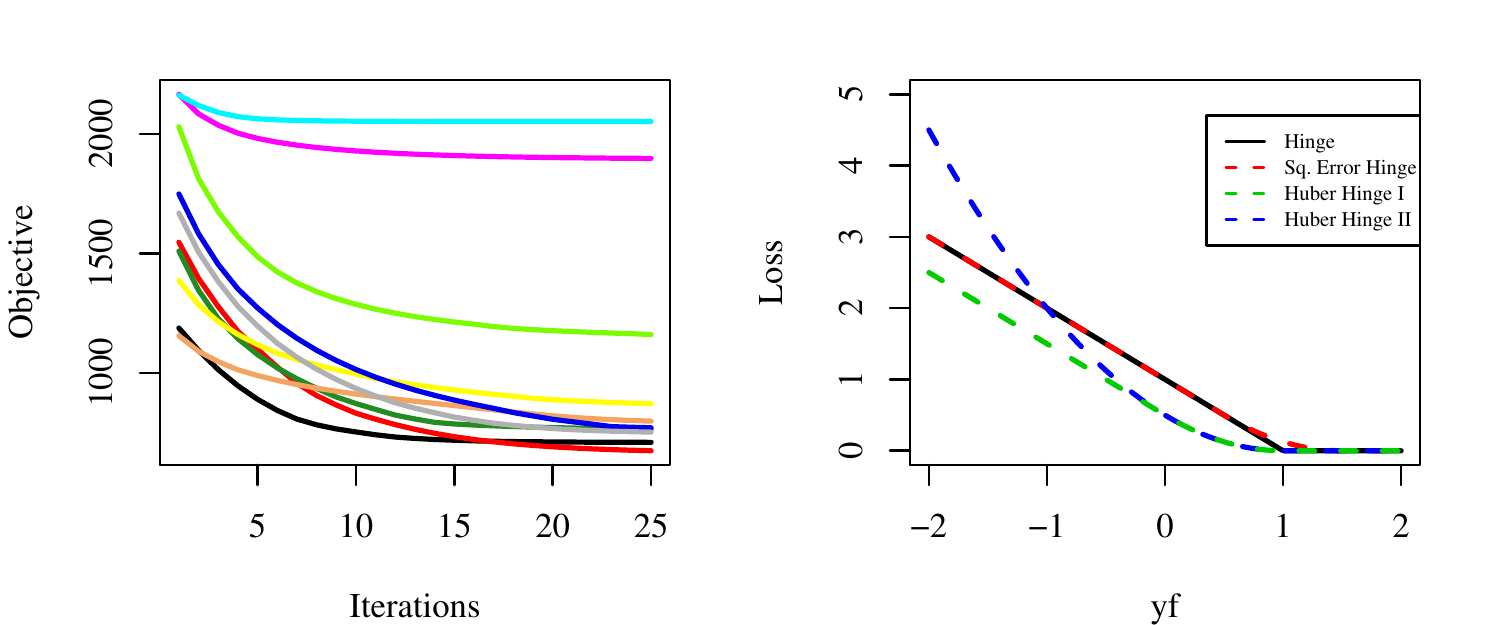}
\caption{\it (a) KNIFE objective for iterations of the KNIFE algorithm
  starting at 10 random set of weights.  Here, a radial kernel with
  squared error loss is used.  (b) Smooth approximations to the
  non-differentiable hinge loss for the support vector machine.}
\label{fig_conv}
\end{center}
\end{figure}

We pause to note the importance of loss function differentiability.  
While many loss functions, such as squared
error and binomial deviance, are continuously differentiable, there is
one notable 
exception, namely the hinge loss of support vector machines.  Because
of the necessity of smoothness conditions for non-convex functions as
proven 
in  \citet{tseng}, the coordinate-wise minimizations of KNIFE
for SVMs may
never converge.   Hence, we suggest using surrogate loss
functions to approximate the hinge loss.
Several smooth versions of this loss have been suggested,
including squared error hinge and a Huberized hinge loss
\citep{hinge}, as shown in Figure \ref{fig_conv} (b).  These both have
been shown to approximate the results 
of the support vector machine well.  Thus, we propose that instead of
using the KNIFE algorithm directly with SVMs, to run the algorithm
with a smoothed version of the hinge loss.  Then, given the set of
optimal weights, calculate the SVM coefficients using the hinge loss.
Our experience reveals that in general, this scheme finds a
reasonable solution which decreases the original objective function
with hinge loss.

While we have discussed KNIFE for linear kernels, we have not given
any results on the convergence of 
KNIFE for non-linear kernels in which we do not find the unique
minimum with respect to the feature weights in each step.  As
previously discussed, for non-linear kernels, the objective is highly
non-convex with respect to the feature weights.  Thus, we cannot claim
any theoretical convergence results for KNIFE with respect to these
kernels.  In numerical examples, however, for differentiable, convex loss
functions and kernels, the KNIFE algorithm converges and decreases the
objective monotonically.  This is shown in an example with radial
kernels and squared error loss in Figure \ref{fig_conv} (a).  

There are several intuitive reasons explaining observed convergence of
KNIFE for convex, differentiable losses and kernels.  First, if the
kernel is convex and differentiable, then the linearization with
respect to the square of the weights is a global under-estimator of the kernel.
Thus, one can surmise that minimizing the KNIFE objective with the
linearized kernel will tend to decrease the objective, except for
pathological cases.  If we restrict the feature weights in each
iteration to be close to previous weights, we know that a linear
approximation is close to the true function if the function is continuous.
For this reason, linearizing the kernel with the squared weights is
crucial since it restricts the search directions for the weight
vector, keeping the weights close to the previous set of weights.
Also, since the minimization with respect to the feature weights is
followed by estimating the coefficients, for which the objective is
convex, the KNIFE algorithm will, in general,
decrease the KNIFE objective. Since the objective is bounded below,
this explains the observed convergence of the KNIFE algorithm in
numerical examples.  We surmise that with possibly stricter conditions
on the objective 
function, stronger theoretical convergence results may be attainable
even for non-linear kernels.  
Overall, while we do not give precise convergence
results, we recommend using KNIFE with convex,
differentiable losses and kernels.

\subsection{KNIFE Feature Path Algorithm}
\label{section_path_alg}

\begin{figure}[!ht]
\begin{center}
    \includegraphics[width=2.25in]{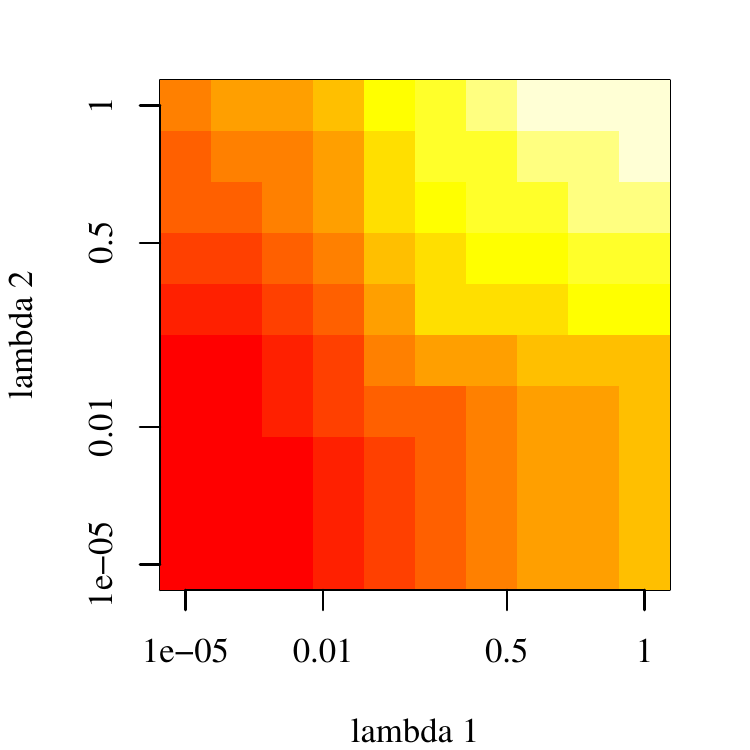}
\caption{ \it Values of the KNIFE objective for several penalty
  parameter, $\lambda_{1}$ and $\lambda_{2}$ values.  Darker blocks
  represent smaller values of the objective. }
\label{fig_params}
\end{center}
\end{figure}

Since the KNIFE method uses regularization to extract important
features, we can modify KNIFE to give a path-wise algorithm resulting
in non-linear feature paths.  Regularization paths are common in
linear regression problems where the values of the coefficients are
given for each value of a penalty parameter used.  Our path algorithm
is similar.  For KNIFE, however, we have two penalty parameters,
$\lambda_{1}$ and $\lambda_{2}$.  Both of these parameters place a
penalty on the feature weights and thus are related.  We explore this
relationship in Figure \ref{fig_params} where we give values of the KNIFE
objective \eqref{true_objective} for various penalty parameter values.

While both parameter values effect the feature weights
and the value of the objective, we have pointed out that the parameter
$\lambda_{2}$ encourages sparsity.  Hence, when formulating a feature
path algorithm, we focus on $\lambda_{2}$, fixing the value of
$\lambda_{1}$.  In general, setting $\lambda_{1} = 1$ (or if the loss
function is given as $\frac{1}{n} L(\Y, f(\X))$, then $\lambda_{1} =
\frac{1}{n}$) performs
reasonably well and is thus our default value for the remainder of the
paper.  Now, setting $\lambda_{2} = 0$ gives
no direct penalty on the feature weights and thus all features are
permitted to be non-zero.  Hence, the path algorithm
varies from $\lambda_{2} = 0$ where all feature weights are non-zero to $\lambda_{2} =
M$, where $M$ is the value at which all weights become zero.  The
path-wise algorithm is given below in Algorithm \ref{alg_knife_path}.

\begin{algorithm}[!ht]
\caption{KNIFE Feature Path Algorithm}
\label{alg_knife_path}
\begin{enumerate}
\item Fix $\lambda_{1}$, set $\lambda_{2} = 0$ and initialize $\alpha$ and $\w$.
\item Fit KNIFE with $\alpha^{(t-1)}$ and $\w^{(t-1)}$ as warm starts.
\item Increase $\lambda_{2}$.
\item Repeat Steps 2-3 until $\w = \mathbf{0}$.
\end{enumerate}
\end{algorithm}

Algorithm \ref{alg_knife_path} maps out feature paths because of two
attributes of the original KNIFE algorithm.  First, recall that in
KNIFE we linearize kernels with respect to the square of the weights,
creating an algorithm that is sticky at zero.  This means that as we
increase $\lambda_{2}$ once a particular feature's weight is set to zero, it
cannot ever become non-zero.  This attribute permits us to efficiently use warm
starts for the coefficients and weights, speeding computational time
considerably.  Also, with warm starts and a small increase in $\lambda_{2}$,
one can use a single update of the coefficients and weights at each
iteration to approximate the feature paths.

We pause briefly to compare this path algorithm to the well known
coefficient paths of the lasso and LAR (Least Angle Regression)
algorithms \citep{lar}.  In both of these regularization paths, the algorithm
begins with no variables in the model and incrementally includes
coefficients who most correlate with the response.  In our KNIFE
path algorithm, however, we begin with all features in the model and
incrementally eliminate the features that are uncorrelated (in the
kernel space) with the response.  Thus, the KNIFE path algorithm can
be thought of as a regularization approach to backwards elimination
for kernels.
Also, the lasso regularization paths permit coefficient paths to cross
zero and enter and re-enter the model.  KNIFE does not allow this
because of the {\it sticky} property of the feature weights, meaning
that one a feature weight is set to zero it cannot move away from zero.  The
KNIFE path algorithm is then like a kernel analog of other common
coefficient regularization paths.

\section{Results}
\label{section_results}

In this section, we explore the performance of the KNIFE algorithm
and the KNIFE path algorithm on both real and simulated data.  We
demonstrate KNIFE in conjunction with two predictions methods, least
squares and support vector machines.  Three of the most common
kernels, the inner product kernel, polynomial kernels, and Gaussian (radial)
kernels are used.  For three simulations, we give results both in
terms of prediction error and feature path realizations, comparing
KNIFE to existing feature selection methods such as Sure Independence
Screening (SIS) \citep{sis} and Recursive Feature Elimination
(RFE) \citep{rfe}.  Finally, we give results on gene selection in colon
cancer microarray data \citep{colon} and feature selection in vowel
recognition data \citep{esl} and Parkinson's disease data \citep{parkinsons}.  

\subsection{Simulations}
\label{section_simulations}

We present three simulation examples to demonstrate the performance of
KNIFE: linear regression, non-linear regression, and non-linear
classification. Each of the three simulations were repeated fifty
times and error rates are averaged.  Training sets were generated of
dimension $100 \times 10$ and test sets were of dimension $1000 \times
10$.  Parameters for KNIFE and all other comparison methods were found
by taking the parameter giving the minimum error on a separate
validation set of dimension $1000 \times 10$. For all KNIFE methods,
$\lambda_{1}$ was fixed at 1 and $\lambda_{2}$ was found by
validation.  To be fair, we validate one parameter for each comparison
method also.  All simulations include
true features are noise features as specified below.

\subsubsection{Linear Regression}

We simulate data from a linear model to compare the KNIFE method using
an inner product kernel with squared error loss to common regression
methods.  Recall that this most basic form of KNIFE given in
\eqref{lin_1} is closely related to many regression methods as
discussed in Section \ref{section_connections}.  This simulation is
then given more to illustrate these connections between regression
methods than to promote the use of KNIFE for linear regression.   

In this simulation, we have ten features, five of which are random
noise.  The true coefficients are then $\beta^{true} =[6, -4, 3,
2 ,-2, 0, 0, 0, 0, 0]^{T}$.  We take the data, $\x \in \Re^{p}$, to be
standard normal and the response is given by the following model.  $y
= \x \beta^{true} + \epsilon$, where $\epsilon \sim N(0,1)$.  The
results in terms of training and test error are given in Table
\ref{tab_linear}.  

\begin{table}[!ht]
\begin{center}
\begin{tabular}{|l|c|c|}
\hline
Method & Training Error & Test Error \\
\hline
Least Squares &    0.9220 (.0026)   & 1.0937 (.0015) \\
\hline
 Ridge &   0.9220 (.0026) &   1.0937 (.0015) \\
\hline
Lasso  &  0.9452  (.0026) &   {\bf  1.0725} (.0014) \\
\hline
Elastic Net &     0.9472 (.0026) &    1.0746 (.0014) \\
\hline
KNIFE &    0.9377 (.0026) &   {\bf 1.0738 }  (.0016) \\
\hline
\end{tabular}
\caption{ \it Simulation results for linear regression.  The data has ten
  features, five of which are noise features.  Mean squared error on
  training and test sets are given with the standard error in
  parenthesis.  The two best performing methods are in bold. }
\label{tab_linear}
\end{center}
\end{table}

These results show that the linear KNIFE method
performs similarly to
other sparse regression methods such as the lasso and elastic net.  

\subsubsection{Non-linear Regression}

To investigate KNIFE with other kernels, we
simulate a sinusoidal, non-linear regression problem.  Here we still use squared
error loss but use KNIFE with radial and second order polynomial
kernels.  We compare KNIFE to linear ridge regression and kernel or
generalized ridge regression.  In addition we give results for the
filtering method SIS and the selection method RFE both used with kernel
ridge regression.  We note here that the scale parameter, $\gamma$,
for the radial kernel is set to $\gamma = 1/p$, a commonly used
default for the non-KNIFE methods using radial kernels.  This
simulation, like the first, also has ten features with five of
them random noise.  The true coefficients and data are the same as in the
previous simulation, but we change the model to be $y =
\mathrm{sin}(\x) \beta^{true} + \epsilon$, thus adding a non-linear
element.  

\begin{table}[!ht]
\begin{center}
\begin{tabular}{|l|c|c|}
\hline
Method & Training Error & Test Error \\
\hline
Ridge &    4.8337 (.0202) &    6.2589 (.0109) \\
\hline
Kernel Ridge  &  0.0874 (.0005) &  6.1239 (.0106) \\
\hline
SIS/ Kernel Ridge  &   0.9592 (.0122) &   {\bf 3.8119} (.0389) \\
\hline
RFE/ Kernel Ridge  &  0.7848 (.0074) &   5.4386 (.0366) \\
\hline
KNIFE/ Kernel Ridge (radial) &   2.1187 (.0058) &  {\bf  3.5498} (.0083) \\
\hline
KNIFE/ Kernel Ridge (polynomial)  &   5.2376 (.0181) &  6.8591 (.0186) \\
\hline
\end{tabular}
\caption{\it Simulation results for sinusoidal, non-linear regression. The data has ten
  features, five of which are noise features.  Mean squared error on
  training and test sets are given with the standard error in
  parenthesis.  The two best performing methods are in bold.   }
\label{tab_l2}
\end{center}
\end{table}

\begin{figure}[!ht]
\begin{center}
    \includegraphics[width=6in]{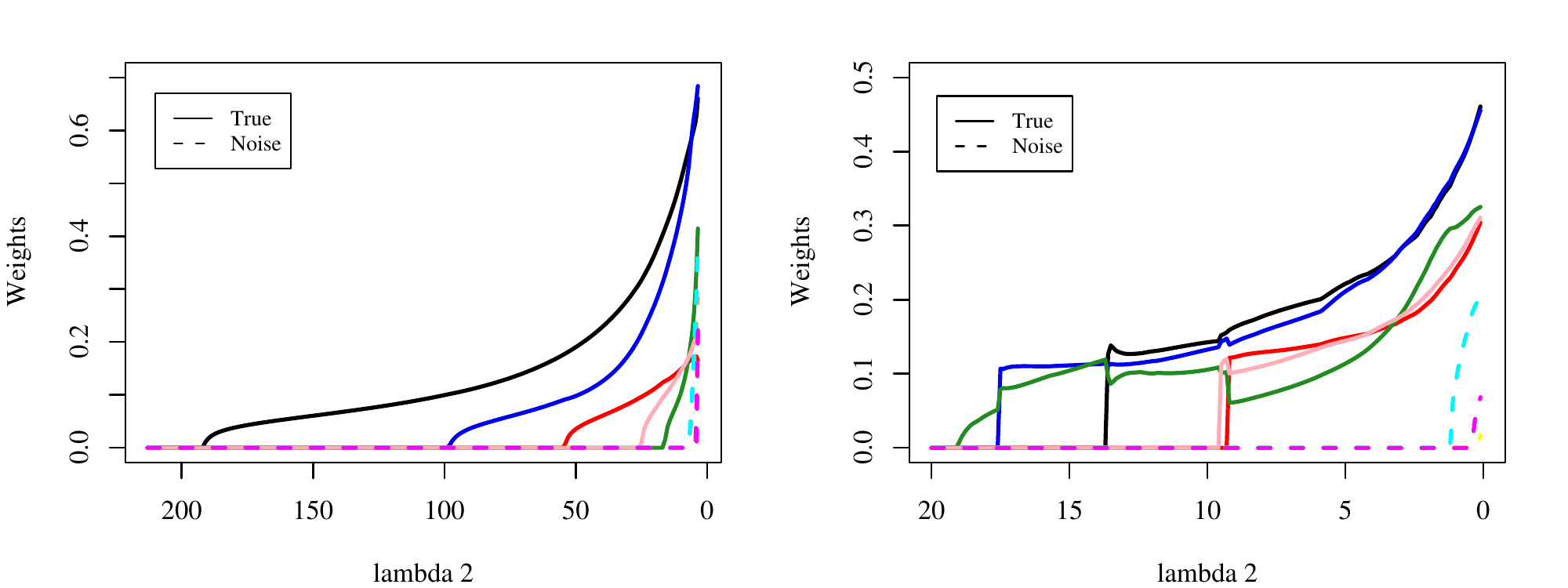}
\caption{ \it Feature path realizations of KNIFE for polynomial kernels of order 2
  (left) and radial kernels (right) with a squared error loss
  function.  Data of dimension $100 \times 10$ was simulated from the
  sinusoidal simulation with 
  five true features and five noise features.  KNIFE for both kernel
  types gives non-zero weights to the five true features for much of
  the feature paths.}
\label{fig_l2_paths}
\end{center}
\end{figure}

In Table \ref{tab_l2}, we report the mean squared error for the
training and test sets over the fifty simulations.  We see that KNIFE
with radial kernels outperforms competing methods including the
commonly used SIS and RFE methods for extracting important features.
In addition, the feature weights of KNIFE are very stable, with only
five total noise features given a non-zero weight for radial kernels out of
the all five noise features in fifty simulations.
Among the KNIFE methods, a radial kernel gives much better error rates than
the second order polynomial, indicating that choosing the wrong kernel
can be problematic.  To investigate this further we give the
entire feature path realizations for both radial and polynomial kernels in Figure
\ref{fig_l2_paths}.  We see that while the polynomial kernel gives
much smoother feature paths, the radial kernel estimates the true
features for a much larger portion of the feature path.

Here we make a brief note about the feature weights in radial
kernels.  In general, the scaling factor $\gamma$ can be extremely
important in radial kernels, but for KNIFE, we do not need to include
any scale factors.  The feature weights themselves act as automatic
scaling factors, adjusting to fit the data.  This is seen in the radial
feature paths of Figure \ref{fig_l2_paths} where all of the feature
weights adjust when one new non-zero feature is added to the model.

\subsubsection{Non-linear Classification}

We use support vector machines to assess KNIFE's performance on
non-linear classification simulations.  The simulation is the skin of
the orange simulation, previously presented as a motivating example, with four true features and six noise features \citep{esl}.  Here, the first class has four standard normal features,
$\x_{1}, \x_{2}, \x_{3}, \x_{4}$,
and the second class has the same conditioned on $9 \leq
\sum_{j=1}^{4} \x_{j}^{2} \leq 16$.  Thus, the model has one class which
is spherical with the second class surrounding the sphere like the
skin of the orange.  

\begin{table}[!ht]
\begin{center}
\begin{tabular}{|l|c|c|}
\hline
Method & Training Error & Test Error \\
\hline
SVM      &   0 (0)  &  {\bf 0.1918} (.0006) \\
\hline
SIS / SVM    &     0 (0) &    0.3689 (.0016) \\
\hline
RFE / SVM  &  0.0206 (.0007) &    0.1937 (.0014) \\
\hline
KNIFE &    0.0468 (.0006) &    {\bf 0.1136} (.0009) \\
\hline
\end{tabular}
\caption{\it Average misclassification errors for the skin of the orange
  simulation with four true features and six noise features.  All
  methods use support vector machines with second 
  order polynomials.  Standard errors are in parenthesis with the two
  best performing methods in bold. }
\label{tab_svm}
\end{center}
\end{table}

The results of this simulation in terms of misclassification error on
he training and test sets are presented in
Table \ref{tab_svm}.  A second order polynomial kernel was used for
all methods.  We note that KNIFE was used with the squared error hinge
loss approximation to the hinge loss of the SVM.  On this simulation
also, KNIFE outperforms both regular SVMs and other common feature
selection methods.  A portion of the feature paths are presented in
Figure \ref{fig_svm_path}.  Here, two out of the four true features
are selected for much of the path.  This is due to the fact that all
four true features are of the same distribution for this simulation.
Here, we also comment on a unique property of KNIFE for support vector
machines.  The SVM is sparse in the observation space, meaning that
only a subset of the observations are chosen as support vectors.  With
KNIFE SVM, we get sparsity in the feature space also, leaving us with
an important  sub-matrix of observations and features
that can limit computational storage for prediction purposes.

\begin{figure}[!ht]
\begin{center}
    \includegraphics[width=3in]{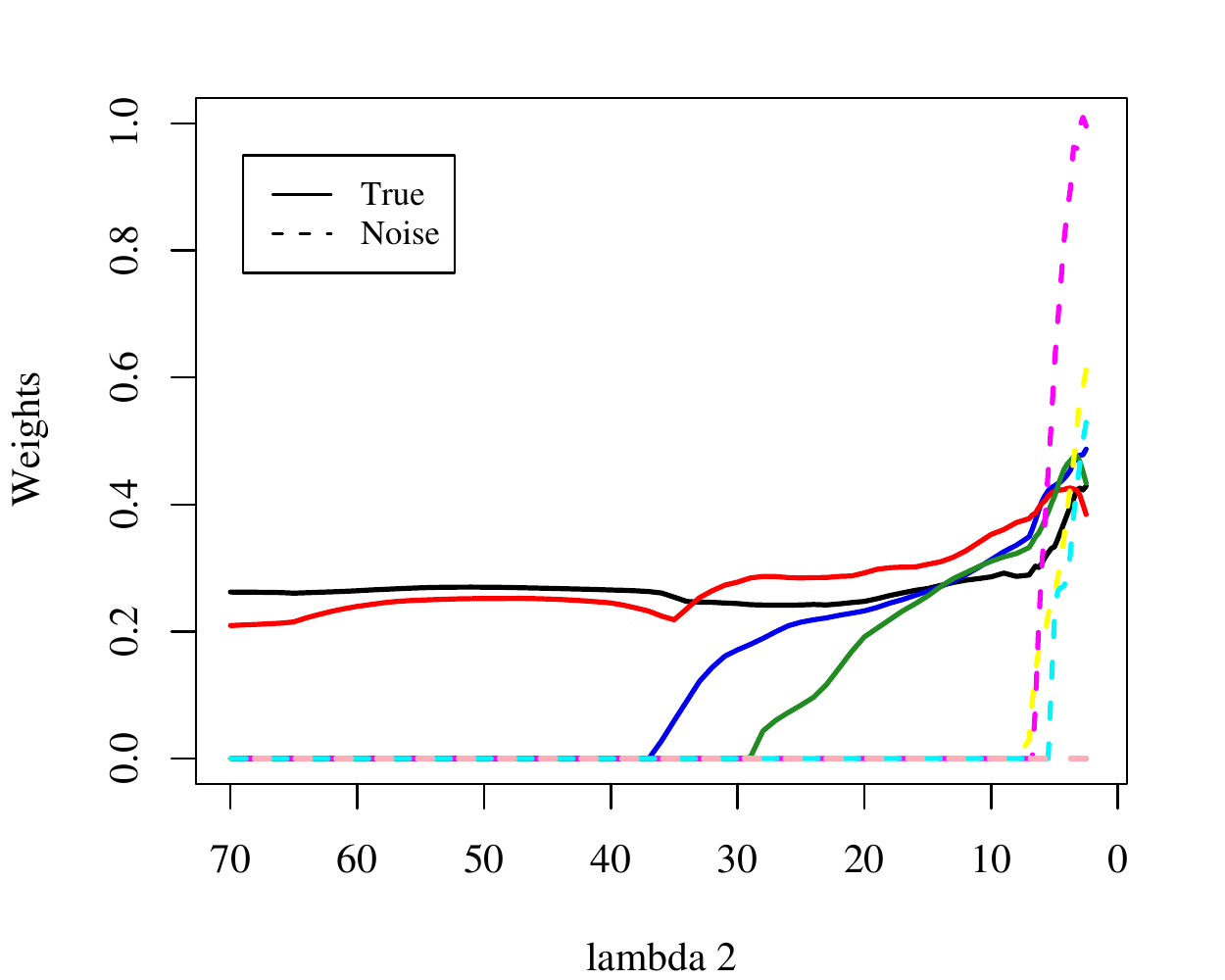}
\caption{ \it Feature path realization of KNIFE for polynomial kernels
  of order 2 for support vector machines (hinge loss).  Data of
  dimension $100 \times 10$ was simulated from the skin of the orange
  simulation with four true features and six noise features.}
\label{fig_svm_path}
\end{center}
\end{figure}

\subsection{Example Data}

Finally, we apply KNIFE to three feature selection applications,
beginning with
microarray data.  With thirty thousand human genes, doctors often need
a small subset of genes to test that are predictive of a disease.  For
this application, we use microarray data on colon cancer given in \citet{colon}.  The dataset consists of 62 samples, 22 of
which are normal and 40 of which are from colon cancer tissues.  The
genes are already pre-filtered, consisting of the 2,000 genes with the highest variance across samples.

\begin{table}[!ht]
\begin{center}
\begin{tabular}{|l|c|c|c|}
\hline
\# Genes  & SIS/SVM & RFE/SVM & KNIFE/SVM \\
\hline
2000  & 0.0355 (.0062)  & 0.0355 (.0062) & 0.0387 (.0062)\\
\hline
500  & 0.0419 (.0027)  & 0.1290 (.0050) & 0.0484 (.0059)\\
\hline
250  & 0.0516 (.0027)  & 0.1452 (.0063) & 0.0516 (.0057)\\
\hline
100  & 0.0613 (.0039)  & 0.1677 (.0066) & 0.0710 (.0054)\\
\hline
50  & 0.0645 (.0043)  & 0.1774 (.0065) & 0.0806 (.0057)\\
\hline
25  & 0.0677 (.0024)  & 0.1742 (.0041) & 0.0871 (.0055)\\
\hline
15  & 0.0742 (.0040)  & 0.1484 (.0049) & 0.1065 (.0059)\\
\hline
10  & 0.0968 (.0046)  & 0.1581 (.0060) & 0.1194 (.0061)\\
\hline
\end{tabular}
\end{center}
\caption{\it Average misclassification rates with standard errors on ten randomly created test
  sets for the colon cancer microarray data.  All methods use a linear
  support vector machine. }
\label{tab_colon}
\end{table}

For this analysis, we use a linear SVM for classification, comparing
KNIFE to gene filtering using SIS and RFE.  We evaluate eight subsets
of previously fixed numbers of genes on the three methods.  For the gene selection with RFE, we begin by eliminating 50 genes, ten genes and
then one gene at each step as outlined in
\citet{rfe}. To
determine predictive ability, we split the samples randomly into
training and test sets of equal sizes.  This is repeated ten times and
misclassification rates are averaged.  These results are given in
Table \ref{tab_colon}.

\begin{table}[!ht]
\begin{center}
\begin{tabular}{|l|c|c|}
\hline
\# Genes  & SIS/SVM  & KNIFE/SVM \\
\hline
2000&    0.5913 &   0.5913\\
\hline
500 &    0.6365  &  0.5969\\
\hline
250 &    0.6451  &  0.5981\\
\hline
100 &    0.6534  &  0.5951\\
\hline
50 &    0.6674   & 0.5964\\
\hline
25 &    0.6652  &  0.5898\\
\hline
 15 &   0.6910   & 0.5946\\
\hline
10 &    0.7139   & 0.5895\\
\hline
\end{tabular}
\end{center}
\caption{\it Median of absolute pair-wise correlations of genes selected
  by each method on the colon cancer microarray data.}
\label{tab_colon_cor}
\end{table}

The results indicate that KNIFE outperforms the commonly used RFE
filtering method for all subsets of genes.  For smaller
subsets, however, SIS filtering performs the best in terms of test
error.  While the subset of genes determined by SIS may be good in
terms of prediction, often researchers are interested in a subset of
genes that are members of different pathways and are hence less
correlated.  For each of the subsets, we report the median of the
absolute pair-wise correlations of the genes selected by both SIS and
KNIFE in Table \ref{tab_colon_cor}.  Here, we notice that KNIFE
selects a group of genes that have around that same median
correlation as the original 2,000 genes.  On the other hand, SIS
filtering tends to select highly correlated groups of genes that might
not be as desirable for research purposes.

The vowel recognition dataset \citet{esl} consists of eleven
classes of vowels broken down into ten features which fifteen
individuals were recorded saying six times.  For illustration, we apply a radial kernel SVM and KNIFE to classify between vowels 'i' and 'I' based on ten
features shown on the left in Figure \ref{fig_vowel}.  We use five fold
cross-validation to determine the margin size for the SVM and the
$\lambda_{2}$ value for KNIFE.  Both methods were trained on a dataset
with 48 instances of each vowel and tested with 42 instances of each.
KNIFE gives a test misclassification error of 8.3\% while the SVM gives an error of
19.1\%.  We see, in Figure \ref{fig_vowel} that KNIFE selects six
features that are indicative of the two vowel types.  For some other pairs
of vowels, however, KNIFE
does not perform better than SVMs indicating that all ten features may
be needed for classification of other vowels.

\begin{figure}[!ht]
\begin{center}
    \includegraphics[width=6in]{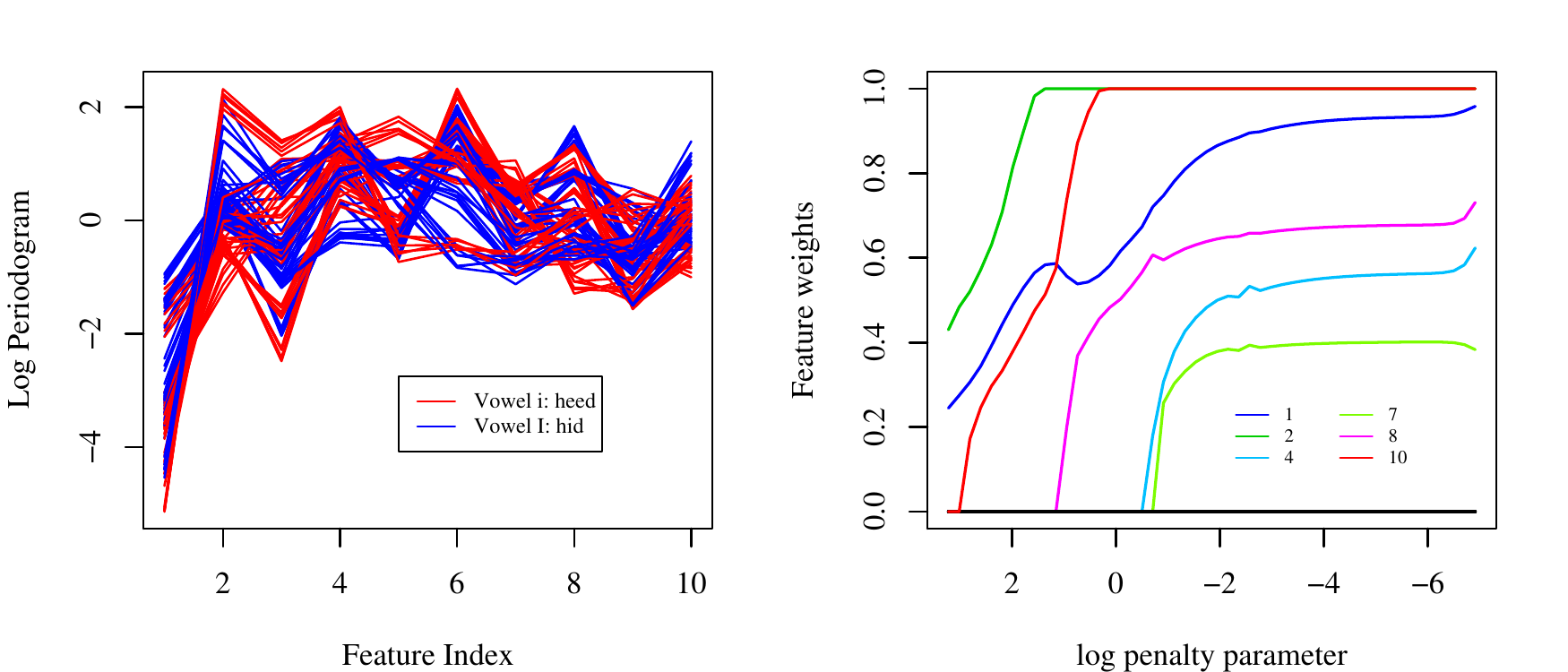}
\caption{\it Vowel recognition data for vowels 'i' and 'I' (left) with
KNIFE feature paths for classifying the two vowels using a support
vector machine with radial kernels (right).  Five fold
cross-validation on the training set chose to include the six
features shown.  KNIFE gives a test misclassification error of 8.3\%,
while a radial 
kernel SVM has 19.1\% test error.}
\label{fig_vowel}
\end{center}
\end{figure}

The Parkinson's disease dataset consists of 22 biomedical voice
measurements from 31 individuals, 23 of which have Parkinson's
disease \citep{parkinsons}.  Five-fold cross-validation by individuals was used to choose
the optimal margin for SVMs and $\lambda_{2}$ for KNIFE.  We then
randomly divided the data (by individuals) into training and test sets
of equal size.  Both KNIFE and the SVM were able to perfectly separate
diseased from healthy individuals in all 100 trials.  KNIFE, however,
chooses an average of eight features for use in the classification.  In
Table \ref{tab_park}, we give the top eleven most frequently selected
features and compare these to the features selected by the model in
\citet{parkinsons}, which were selected by filtering and
assessing all possible feature combinations.  Several of the features measure
similar attributes and thus are not often selected together.  Three
features, however, were selected in all 100 models and given a large weight.

\begin{table}[!h]
\begin{tabular}{|l|c|c|c|c|}
\hline
Feature & Explanation & Average  & Times & Selected by
\\
 & & Weight & Selected & Little {\it et al.} \\
\hline
DFA & Signal fractal scaling exponent & 0.902 & 100/100 & Yes \\
MDVP:Fo(Hz) & Average fundamental frequency & 0.879 & 100/100 &
No \\
RPDE & Dynamical complexity measure & 0.803 & 100/100 & Yes
\\
HNR & Ratio to noise of tonal components & 0.669  & 87/100  & Yes \\
spread 2 & Fundamental frequency variation & 0.668 &  83/100 & No \\
spread 1 &  Fundamental frequency variation & 0.388 & 72/100 & No \\
D2 & Dynamical complexity measure & 0.168 & 68/100 &   Yes \\
MDVP:Flo(Hz) & Minimum fundamental frequency & 0.133 & 58/100 & No \\
MDVP:RAP  & Fundamental frequency variation & 0.229 & 33/100 & No \\
PPE &  Fundamental frequency variation &  0.294 & 32/100 & Yes \\
Shimmer:APQ3 & Variation in amplitude & 0.106 & 16/100 & No \\
\hline
\end{tabular}
\caption{\it Features, or biomedical voice measurements, selected by
  KNIFE with radial kernel SVMs for classification of healthy
  individuals from those with 
  Parkinson's disease.  The average weight and the number of times out
  of 100 which KNIFE selects the feature are given.   We also report
  whether the feature was selected by
  \citet{parkinsons} for use in their radial SVM classifier.}
\label{tab_park}
\end{table}

\section{Discussion}
\label{section_discussion}

We have presented a method for selecting important features with
non-linear kernel regression and classification methods: KerNel
Iterative Feature Extraction 
(KNIFE).  The KNIFE optimization problem forms feature weighted
kernels and seeks to find the minimum of a penalized, feature weighted
kernel loss function 
with respect to both the coefficients and the feature weights.  We
have given the KNIFE algorithm which iteratively finds the
coefficients, linearizes the kernels, then finds the set of feature
weights.  This algorithm, under broad conditions, converges and
decreases the KNIFE objective for each iteration.  A path-wise
algorithm is also given for the kernel feature weights.  Finally, we
have demonstrated the utility of KNIFE for feature selection and kernel
prediction for several example simulations and microarray data.

Computationally, the KNIFE algorithm compares favorably to existing
kernel feature selection methods, which the exception of simple
feature filtering methods including SIS.  The KNIFE algorithm
iterates between an optimization problem in the $n$-dimensional
feature space and then a $p$-dimensional feature space. 
 For comparison, RFE solves a 
problem in $n$-dimensional space several times, checking each
remaining feature at each iteration.  Thus, several kernels must be
computed for each iteration of RFE.  Also, for support vector
machines, existing methods such as the Radius-Margin bound are
computed approximately using computationally intensive
conjugate-gradient methods \citep{weston}.  In addition, since the
KNIFE algorithm generally decreases the objective, one can stop the
algorithm after a few iterations to limit computational costs.

For the KNIFE algorithm with particular regression or classification
problems, we did not give problem specific KNIFE algorithms.  With
squared error loss, for example, finding the coefficients is simply
performing kernel ridge regression, while finding the feature weights
is performing a non-negative penalized kernel least squares.  For
the squared error hinge loss approximation to the support vector
machine, finding the feature weights amounts to non-negative
projected, penalized kernel least squares.  These problem specific
algorithms deserve further investigation and are left the future
work.  In addition, algorithms applicable to high-dimensional settings
are needed.

The KNIFE technique for kernel feature selection is applicable to a
variety of kernels and regression and classification problems,
specifically methods with convex and differentiable loss functions and
kernels.  Also, KNIFE can be used with non-differentiable loss
functions if a surrogate smoothed version of the loss is used for the
iterative steps as demonstrated with SVMs.  This broad applicability
of KNIFE means it can be used in conjunction with most kernel
regression and classification problems, including kernel logistic
regression which has a binomial deviance loss.  In addition, KNIFE may
be modified to work with kernel principal
component analysis and kernel discriminant analysis (also kernel
canonical correlation) which can be written with a Frobenius norm
loss.  Thus, the KNIFE method has many potential future uses for
feature selection in a variety of kernel methods.

\section{Acknowledgments}

We would like thank Robert Tibshirani for the helpful suggestions and
guidance in developing and testing this method.
We would also like to thank Stephen Boyd for suggesting kernel
convexification and Trevor
Hastie for the helpful suggestions.

\bibliography{knife}

\begin{thebibliography}{23}
\providecommand{\natexlab}[1]{#1}
\providecommand{\url}[1]{\texttt{#1}}
\expandafter\ifx\csname urlstyle\endcsname\relax
  \providecommand{\doi}[1]{doi: #1}\else
  \providecommand{\doi}{doi: \begingroup \urlstyle{rm}\Url}\fi

\bibitem[Alon et~al.(1999)Alon, Barkai, Notterman, Gish, Ybarra, Mack, and
  Levine]{colon}
U.~Alon, N.~Barkai, D.~A. Notterman, K.~Gish, S.~Ybarra, D.~Mack, and A.~J.
  Levine.
\newblock Broad patterns of gene expression revealed by clustering analysis of
  tumor and normal colon tissues probed by oligonucleotide arrays.
\newblock \emph{Proceedings of the National Academy of Sciences}, 96\penalty0
  (12):\penalty0 6745--6750, June 1999.

\bibitem[Argyriou et~al.(2006)Argyriou, Hauser, Micchelli, and Pontil]{dc_alg}
A.~Argyriou, R.~Hauser, C.~A. Micchelli, and M.~Pontil.
\newblock A dc-programming algorithm for kernel selection.
\newblock In \emph{ICML '06: Proceedings of the 23rd international conference
  on Machine learning}, pages 41--48, 2006.

\bibitem[Breiman(1995)]{nng}
L.~Breiman.
\newblock Better subset regression using the nonnegative garrote.
\newblock \emph{Technometrics}, 37\penalty0 (4):\penalty0 373--384, 1995.

\bibitem[Cao et~al.(2007)Cao, Shen, Sun, Yang, and Chen]{cao}
B.~Cao, D.~Shen, J.~Sun, Q.~Yang, and Z.~Chen.
\newblock Feature selection in a kernel space.
\newblock In \emph{ICML '07: Proceedings of the 24th international conference
  on Machine learning}, pages 121--128, 2007.

\bibitem[Efron et~al.(2004)Efron, Hastie, Johnstone, and Tibshirani]{lar}
B.~Efron, T.~Hastie, I.~Johnstone, and R.~Tibshirani.
\newblock Least angle regression.
\newblock \emph{Annals of Statistics}, 32:\penalty0 407--499, 2004.

\bibitem[Fan and Lv(2008)]{sis}
J.~Fan and J.~Lv.
\newblock Sure independence screening for ultrahigh dimensional feature space.
\newblock \emph{Journal Of The Royal Statistical Society Series B}, 70\penalty0
  (5):\penalty0 849--911, 2008.

\bibitem[Grandvalet and Canu(2002)]{grandvalet}
Y.~Grandvalet and S.~Canu.
\newblock Adaptive scaling for feature selection in svms.
\newblock In \emph{Advances in Neural Information Processing Systems 15}, 2002.

\bibitem[Guyon(2003)]{guyon_multivariate}
I.~Guyon.
\newblock Multivariate nonlinear feature selection with kernel multiplicative
  updates and gram-schmidt relief.
\newblock In \emph{BISC FLINT-CIBI 2003 workshop}, 2003.

\bibitem[Guyon et~al.(2002)Guyon, Weston, Barnhill, and Vapnik]{rfe}
I.~Guyon, J.~Weston, S.~Barnhill, and V.~Vapnik.
\newblock Gene selection for cancer classification using support vector
  machines.
\newblock \emph{Machine Learning}, 2002.

\bibitem[Hastie et~al.(2001)Hastie, Tibshirani, and Friedman]{esl}
T.~Hastie, R.~Tibshirani, and J.~Friedman.
\newblock \emph{Elements of Statistical Learning}.
\newblock Springer New York Inc., 2001.

\bibitem[Li et~al.(2006)Li, Yang, and Xing]{fvm}
F.~Li, Y.~Yang, and E.~Xing.
\newblock From lasso regression to feature vector machine.
\newblock In \emph{Advances in Neural Information Processing Systems 18}, pages
  779--786, 2006.

\bibitem[Lin and Zhang(2007)]{COSSO}
Y.~Lin and H.~H. Zhang.
\newblock Component selection and smoothing in multivariate nonparametric
  regression.
\newblock \emph{Annals of Statistics}, 34\penalty0 (5):\penalty0 2272--2297,
  February 2007.

\bibitem[Little et~al.(To Appear)Little, McSharry, Hunter, Spielman, and
  Ramig]{parkinsons}
M.~A. Little, P.~E. McSharry, E.~J. Hunter, J.~Spielman, and L.~O. Ramig.
\newblock Suitibiliy of dysphonia measurements for telemonitoring of
  parkinson's disease.
\newblock \emph{Biomedical Enginerring, IEEE Transactions on}, To Appear.

\bibitem[Navot and Tishby(2004)]{margin_feat}
A.~Navot and N.~Tishby.
\newblock Margin based feature selection - theory and algorithms.
\newblock In \emph{International Conference on Machine Learning (ICML}, pages
  43--50, 2004.

\bibitem[Neumann et~al.(2005)Neumann, Schn\"{o}rr, and Steidl]{neumann}
J.~Neumann, C.~Schn\"{o}rr, and G.~Steidl.
\newblock Combined svm-based feature selection and classification.
\newblock \emph{Machine Learning}, 61:\penalty0 129--150, 2005.

\bibitem[Tseng(2001)]{tseng}
P.~Tseng.
\newblock Convergence of a block coordinate descent method for
  nondifferentiable minimization.
\newblock \emph{Journal Optimization Theory and Applications}, 109\penalty0
  (3):\penalty0 475--494, 2001.

\bibitem[Wahba et~al.(1999)Wahba, Lin, and Zhang]{wahba}
G.~Wahba, Y.~Lin, and H.~Zhang.
\newblock Generalized approximate cross validation for support vector machines,
  or, another way to look at margin-like quantities.
\newblock Technical report, University of Wisconsin, 1999.

\bibitem[Wang(2008)]{wang}
L.~Wang.
\newblock Feature selection with kernel class separability.
\newblock \emph{Pattern Analysis and Machine Intelligence, IEEE Transactions
  on}, 30\penalty0 (9):\penalty0 1534--1546, 2008.

\bibitem[Wang et~al.(2008)Wang, Zhu, and Zou]{hinge}
L.~Wang, J.~Zhu, and H.~Zou.
\newblock Hybrid huberized support vector machines for microarray
  classification and gene selection.
\newblock \emph{Bioinformatics}, 24\penalty0 (3):\penalty0 412--419, 2008.

\bibitem[Weston et~al.(2000)Weston, Mukherjee, Chapelle, Pontil, Poggio, and
  Vapnik]{weston}
J.~Weston, S.~Mukherjee, O.~Chapelle, M.~Pontil, T.~Poggio, and V.~Vapnik.
\newblock Feature selection for svms.
\newblock In \emph{Advances in Neural Information Processing Systems 13}, pages
  668--674, 2000.

\bibitem[Zhu et~al.(2003)Zhu, Rosset, Hastie, and Tibshirani]{l1svm}
J.~Zhu, S.~Rosset, T.~Hastie, and R.~Tibshirani.
\newblock 1-norm support vector machines.
\newblock In \emph{Neural Information Processing Systems}, page~16, 2003.

\bibitem[Zou(2006)]{adalasso}
H.~Zou.
\newblock The adaptive lasso and its oracle properties.
\newblock \emph{Journal of the American Statistical Association}, 101:\penalty0
  1418--1429, December 2006.

\bibitem[Zou and Hastie(2005)]{el_net}
H.~Zou and T.~Hastie.
\newblock Regularization and variable selection via the elastic net.
\newblock \emph{Journal of the Royal Statistical Society B}, 67:\penalty0
  301--320, 2005.

\end{thebibliography}

\end{document}